\documentclass[letterpaper, 10 pt, conference]{ieeeconf}  

\IEEEoverridecommandlockouts                              

\usepackage{amsfonts}
\usepackage{amsmath}
\usepackage{amsthm}
\usepackage{bm}
\usepackage{booktabs}
\usepackage{diagbox}
\usepackage{dsfont} 
\let\labelindent\relax  
\usepackage{enumitem}
\usepackage{etoolbox}
\usepackage{hyperref}  
\usepackage{mathtools}
\usepackage{multirow}
\usepackage{nicefrac}
\usepackage[final]{pdfpages}
\usepackage{siunitx}
\usepackage{stfloats}  
\usepackage[caption=false]{subfig}
\usepackage{tabu}
\usepackage{tikz}
\usepackage{xcolor}
\usepackage{xfrac}

\usepackage{pgfplots}   
\pgfplotsset{compat=newest}
\usepgfplotslibrary{groupplots}
\usepgfplotslibrary{fillbetween}

\definecolor{pltBlue}{HTML}{1f77b4}
\definecolor{pltOrange}{HTML}{ff7f0e}
\definecolor{pltGreen}{HTML}{2ca02c}
\definecolor{pltRed}{HTML}{d62728}
\definecolor{pltPurple}{HTML}{9467bd}
\definecolor{pltBrown}{HTML}{8c564b}
\definecolor{pltPink}{HTML}{e377c2}
\definecolor{pltGray}{HTML}{7f7f7f}
\definecolor{pltOlive}{HTML}{bcbd22}
\definecolor{pltCyan}{HTML}{17becf}

\graphicspath{{images/}}
\hypersetup{colorlinks=false}

\definecolor{githubColor}{HTML}{2EA44F}

\definecolor{newGray}{HTML}{808080}
\definecolor{colorCircle}{HTML}{0072BD}
\DeclareRobustCommand\sensorcircle{\tikz \fill[black, fill=colorCircle] circle (0.75ex);}

\definecolor{colorRect}{HTML}{D95319}
\DeclareRobustCommand\sensorrect{\tikz \fill[black, fill=colorRect] rectangle (1.5ex, 1.5ex);}

\newcolumntype{O}[1]{S[detect-weight, mode=text, table-format=#1]}

\renewcommand{\bfseries}{\fontseries{b}\selectfont}
\robustify\bfseries
\newrobustcmd{\B}{\bfseries}

\newcommand\copyrighttext{\footnotesize \textcopyright~2023 IEEE. Personal use of this material is permitted. Permission from IEEE must be obtained for all other uses, in any current or future media, including reprinting/republishing this material for advertising or promotional purposes, creating new collective works, for resale or redistribution to servers or lists, or reuse of any copyrighted component of this work in other works.%
}

\newcommand\copyrightnotice{%
    \begin{tikzpicture}[remember picture,overlay]%
 	\node[anchor=south, xshift=0pt, yshift=10pt] at (current page.south)%
 	{\fbox{\parbox{\dimexpr\textwidth-\fboxsep-\fboxrule\relax}{\copyrighttext}}};%
 	\end{tikzpicture}%
}

\newtheoremstyle{tstyle}
  {}
  {}
  {\itshape}
  {}
  {\bfseries}
  {.}
  { }
  {\thmname{#1}\thmnumber{ #2}\thmnote{ (#3)}}%
\theoremstyle{tstyle}

\newtheorem{corollary}{Corollary}
\newtheorem{proposition}{Proposition}

\newcommand{\mbeq}{\overset{!}{=}}

\newcommand{\mat}[1]{\boldsymbol{#1}}
\newcommand{\quat}[1]{\mathrm{#1}}
\newcommand{\quatsv}[2]{\left[#1, #2\right]}
\renewcommand{\vec}[1]{\boldsymbol{#1}}

\newcommand{\mneg}{^\text{\rmfamily \textup{-}}}
\newcommand{\mpos}{^\text{\rmfamily \textup{+}}}
\newcommand{\norm}[1]{\left\lVert#1\right\rVert}

\newcommand{\trans}{^\text{\rmfamily \textup{T}}}

\newcommand{\vectorize}{\operatorname{vec}}

\newcommand{\sol}[1]{\hat{#1}}

\newcommand{\pspace}{\,}  

\title{\LARGE \bf
User Feedback and Sample Weighting\\for Ill-Conditioned Hand-Eye Calibration
}

\author{Markus Horn, Thomas Wodtko, Michael Buchholz, and Klaus Dietmayer%
\thanks{This work was financially supported by the Federal Ministry of Education and Research (BMBF) (project UNICARagil, FKZ\,16EMO0290) and the State Ministry of Economic Affairs Baden-Württemberg (project U-Shift\,II, AZ\,3-433.62-DLR/60).}%
\thanks{All authors are with the Institute of Measurement, Control and Microtechnology, Ulm University, Albert-Einstein-Allee 41, 89081 Ulm, Germany {\tt\footnotesize \{firstname\}.\{lastname\}@uni-ulm.de}}%
}

\begin{document}

\maketitle
\thispagestyle{empty}
\pagestyle{empty}

\begin{abstract}
Hand-eye calibration is an important and extensively researched method for calibrating rigidly coupled sensors, solely based on estimates of their motion.
Due to the geometric structure of this problem, at least two motion estimates with non-parallel rotation axes are required for a unique solution.
If the majority of rotation axes are almost parallel, the resulting optimization problem is ill-conditioned.
In this paper, we propose an approach to automatically weight the motion samples of such an ill-conditioned optimization problem for improving the conditioning.
The sample weights are chosen in relation to the local density of all available rotation axes.
Furthermore, we present an approach for estimating the sensitivity and conditioning of the cost function, separated into the translation and the rotation part.
This information can be employed as user feedback when recording the calibration data to prevent ill-conditioning in advance.
We evaluate and compare our approach on artificially augmented data from the KITTI odometry dataset.
\end{abstract}

\section{INTRODUCTION}

\copyrightnotice%
Hand-eye calibration is a common approach for extrinsic calibration of sensors, e.g., cameras or inertial measurement units (IMUs), especially when target-based calibration is not possible.
The goal of hand-eye calibration is to estimate the transformation between two rigidly coupled coordinate frames based on multiple motion pairs given in the respective frames.
For estimating the calibration using these pairs, also called samples, different methods have been developed, e.g., \cite{daniilidis1999hand, andreff2001robot, brookshire2013extrinsic, horn2021online}.
A common example from robotics is the calibration between an end effector and a camera mounted on the 
robot~\cite{tsai1989new, schmidt2003robust}.
Another example is the calibration of sensors for automated vehicles, e.g., when working with IMUs, which have no field of view (FOV) in which targets can be detected.
In this work, the focus is on the calibration of autonomous vehicles and the special requirements associated with them.

Due to the geometry of hand-eye calibration, the applied motions must fulfill certain constraints.
More precisely, at least two motions with non-parallel rotation axes are required for obtaining a unique solution~\cite{tsai1989new, schmidt2003robust}.
Otherwise, the translation along the common rotation axis cannot be estimated.
Further, \cite{schmidt2003robust} describes that in case the majority of rotation axes are almost parallel, the optimization problem is ill-conditioned, which can lead to a unique but inaccurate solution.
This is an important aspect, especially for autonomous vehicles, which mostly move in a planar manner.

\begin{figure}
    \centering
    \resizebox{0.7\columnwidth}{!}{%
        \input{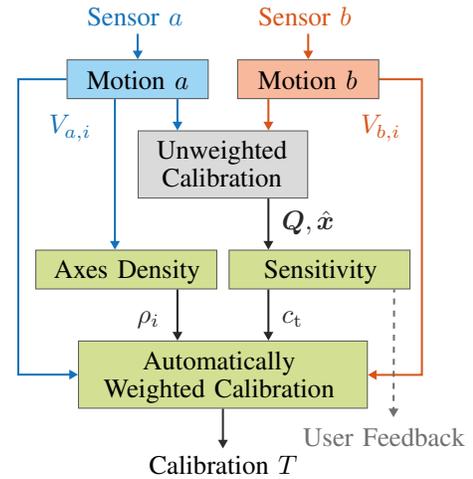}%
    }
    \caption{
    Our proposed method consists of three novel parts~(green):
    The estimated motions $V_{a,i}$ and $V_{b,i}$ are first used for an unweighted calibration, e.g., by \cite{horn2021online}.
    The resulting unweighted cost matrix $\mat{Q}$ and the estimated calibration $\sol{\vec{x}}$ are then utilized for estimating the sensitivity of the cost function and the translation condition number $c_t$.
    The sensitivity and conditioning can be employed as user feedback during calibration.
    Besides, the motions of one of the sensors ($V_{a,i}$ in this case) are used to estimate the local densities $\rho_i$ of the rotation axes.
    These densities are then combined with $c_t$ for generating sample weights and the final calibration $T$.}
    \label{fig:overview_graph}
\end{figure}

Based on the observations in~\cite{tsai1989new}, Shi et al.~\cite{shi2005approach} propose the ``Golden Rules'' for hand-eye calibration:
\begin{enumerate}[topsep=0pt]
    \item Try to make the angles between the motions' rotation axes large.  
    \item Try to make the motions' rotation angles large.  
    \item Try to make the motions' translations small.  
\end{enumerate}
The first rule tackles the previously described issue of an ill-conditioned optimization problem.
The other two rules try to improve the signal-to-noise ratio of the underlying data.

One possibility to implement these rules is when recording the data.
However, especially for the first rule, it is not always clear to the user, e.g., how even the distribution of rotation axes currently is.
Hence, an automatic analysis of the current conditioning and of a possibly overrepresented axis would be beneficial as user feedback.
Such an analysis is a key contribution of our work.

In case data have already been recorded ill-conditioned, using only a subset of samples or assigning individual weights to samples is another option for compensating an overrepresented rotation axis.
However, as our discussion in Section~\ref{sec:related_work} and evaluation in Section~\ref{sec:experiments} will show, most existing approaches for this are either difficult to parameterize or unsuitable for the special case of autonomous vehicles.
Especially the third rule cannot be obeyed with usual vehicles since their rotation is physically coupled to their translation.
Therefore, an adapted approach is required for this case.

Hence, we propose the approach shown in Fig.~\ref{fig:overview_graph} including
\begin{itemize}[topsep=0pt]
    \item a sensitivity measure providing an estimation for the rotation and translation conditioning (cf. Section~\ref{sec:sensitivity});
    \item a formulation for the local density of rotation axes (cf. Section~\ref{sec:density_weights}); and
    \item a sample weighting based on the original conditioning and the axes-density for an improved conditioning of the optimization problem (Section~\ref{sec:combined_weights}).
\end{itemize}

\section{RELATED WORK}
\label{sec:related_work}

Over the last decades, many different approaches for hand-eye calibration have been proposed.
Common approaches are, e.g., linear formulations like~\cite{daniilidis1999hand} or~\cite{andreff2001robot}, based on dual quaternions~(DQs) and homogeneous matrices, respectively.
The paper~\cite{giamou2019certifiably} as well as our previous work~\cite{horn2021online} propose Quadratically Constrained Quadratic Program (QCQP) formulations based on these linear formulations, which enable to obtain a certifiably globally optimal solution.
The evaluations show that the QCQP approaches are more robust to noise than the original linear formulations.
Even though both approaches enable individual weighting of samples in their cost function, the selection of these weights is not described in detail.

Based on the ``Golden Rules'', \cite{shi2005approach} proposes to combine multiple transformations until a minimum angle between rotation axes, as well as a minimum rotation angle, is achieved.
Unfortunately, this is not suitable when using motion data from per-sensor ego-motion estimates, since 
the drift of the estimation leads to larger errors when combining consecutive measurements.
For single transformations,~\cite{zhang2005adaptive} proposes an approach for selecting the ``Golden Rule'' thresholds automatically from available data.
However, as stated in the previous section, strictly following the ``Golden Rules'' is not possible when using motion data from sensors mounted on automated vehicles. 

Other approaches for the selection of samples, like~\cite{li2018simultaneous, liu2019robust}, are designed for removing outliers, e.g., using RANSAC~\cite{li2018simultaneous}.
Further, instead of enforcing hard thresholds and removing samples,~\cite{mair2011spatio, samant2019robust} propose methods for selecting individual weights for each sample.
However, since all mentioned approaches focus on estimation errors or other properties of individual samples, they cannot solve the problem of an uneven sample distribution within the whole dataset by design. 
Nevertheless, it is possible to combine these methods with the approach described in this work to achieve a higher robustness against outliers.
 
The adverse influence of unevenly distributed motion, i.e., a dominant rotation axis, is a key aspect of the methods proposed in~\cite{schmidt2003robust, schmidt2008data}.
Both approaches use the DQ-based formulation from~\cite{daniilidis1999hand} for the underlying optimization problem.
First, both methods pre-select samples with a minimum rotation angle and a maximum translation, following the second and third of the ``Golden Rules''.
Afterward, the motions of one sensor are used to select a sample subset with more evenly distributed rotation axes.
In~\cite{schmidt2003robust}, multiple pairs of samples with small scalar products between their rotation axes are selected using brute force.
Even though this ensures a more even distribution, the complexity is $\mathcal{O}(n^4)$, which leads to high run times for larger datasets.
Therefore, \cite{schmidt2008data} proposes a selection of samples based on the codebook resulting from a vector quantization~(VQ) of the rotation.
The closest sample to each codebook entry is selected for calibration.
Further, they compare their performance on different rotation representations, like rotation axes or quaternions.
However, a major drawback of this approach is its dependency on the performance of the VQ approach and the selected codebook size.
Even though~\cite{schmidt2008data} proposes a relative codebook size that depends on the dataset size, our experiments in Section~\ref{sec:experiments:parameters} show that it is difficult to find an optimal size.

In contrast, our method directly calculates the local density of the rotation axes for each sample without relying on an additional preprocessing step.
This density calculation is easier to parameterize, as the experiments in Section~\ref{sec:experiments} show.
Furthermore, in~\cite{schmidt2008data}, many samples are removed even though they could have a positive influence on the calibration result, especially when the estimated motions are subject to higher noise.
To avoid this, our proposed method assigns individual weights to each sample.
This means that samples with a higher local density are only weighted lower instead of removing them completely, so they can still contribute to the overall calibration result.

\section{FOUNDATIONS}
\label{sec:foundations}

In this section, a short introduction into quaternions and DQs for spatial transformations is given.
Further, the optimization problem for hand-eye calibration derived in~\cite{horn2021online} is described, which is the foundation for this work.
Transformations without a fixed representation are denoted as functions~$T$, (dual) quaternions are denoted as upright letters $\quat{q}$ with the respective vectorized form as $\vec{q} = \vectorize(\quat{q})$.

\subsection{Dual Quaternions}
\label{sec:dual_quaternions}

A common representation for spatial rotations are unit quaternions~\cite{siciliano2009robotics}.
A quaternion $\quat{r}$ can be denoted in scalar-vector notation $\quat{r} = \quatsv{s_\quat{r}}{\vec{v}_\quat{r}}$ using the real-valued scalar $s_\quat{r} \in \mathds{R}$ and the imaginary parts, denoted by $\vec{v}_\quat{r} \in \mathds{R}^3$.
The multiplication of two quaternions yields
\begin{align}
    \quat{a} \, \quat{b} = \quatsv
        {s_\quat{a} s_\quat{b} - \vec{v}_\quat{a}\trans \vec{v}_\quat{b}}
        {\ s_\quat{a} \vec{v}_\quat{b} + s_\quat{b} \vec{v}_\quat{a} + \vec{v}_\quat{a} \times \vec{v}_\quat{b}} \pspace .
\end{align}

A dual quaternion $\quat{q} = \quat{r} + \epsilon \, \quat{d}$ consists of the real part $\quat{r}$, the dual part $\quat{d}$ and the dual unit $\epsilon$ with $\epsilon^2 = 0$ \cite{mccarthy1990introduction}.
The multiplication of two DQs yields
\begin{align}
    \quat{q}_a \, \quat{q}_b =  \quat{r}_a \quat{r}_b + \epsilon \, (\quat{r}_a \quat{d}_b + \quat{d}_a \quat{r}_b) \pspace .
\end{align}

A unit quaternion $\quat{r} = \quatsv{\cos\left(\nicefrac{\varphi}{2}\right)}{\vec{n} \sin\left(\nicefrac{\varphi}{2}\right)}$ can be used to describe a 3D rotation by an angle $\varphi \in \mathds{R}$ around an axis $\vec{n} \in \mathds{R}^3$ with $\norm{\vec{n}}_2 = 1$.
Thus, $\quat{r}$ and $-\quat{r}$ describe the same rotation.
For spatial transformations, a unit DQ can be used.
In this case, the real part $\quat{r}$ is a unit quaternion describing the rotation, and the dual part is $\quat{d} = \frac{1}{2} \quatsv{0}{\vec{t}} \quat{r}$ with translation vector $\vec{t} \in \mathbb{R}^3$.
Consequently, pure translational and pure rotational transformations are described by the functions
\begin{subequations}
\begin{align}
    \label{eq:pure_translation}
    T_{t}(\vec{t}) &= \quatsv{1}{\vec{0}} + \epsilon \, \frac{1}{2} \quatsv{0}{\vec{t}} \pspace ,\\
    \label{eq:pure_rotation}
    T_{r}(\varphi, \vec{n}) &= \quatsv{\cos\left(\frac{\varphi}{2}\right)}{\vec{n} \sin\left(\frac{\varphi}{2}\right)} + \epsilon \quatsv{0}{\vec{0}} \pspace ,
\end{align}
\end{subequations}
respectively.

The transformation chain $T_a \circ T_b$ represented as DQs yields $\quat{q}_a \,\quat{q}_b$.
Using vectorized DQs, $\quat{q}_{ab} = \quat{q}_a \, \quat{q}_b$ can be represented as matrix-vector products $\vec{q}_{ab} = \mat{Q}\mpos_a \vec{q}_b = \mat{Q}\mneg_b \vec{q}_a$ with 
\begin{align}
    \mat{Q}\mpos_a = \begin{bmatrix}
        \mat{R}\mpos_a & \mat{0}\\
        \mat{D}\mpos_a & \mat{R}\mpos_a
    \end{bmatrix} \pspace , \quad
    \mat{Q}\mneg_b = \begin{bmatrix}
        \mat{R}\mneg_b & \mat{0}\\
        \mat{D}\mneg_b & \mat{R}\mneg_b
    \end{bmatrix} \pspace ,
\end{align}
where $\mat{R}\mpos_a$, $\mat{R}\mneg_b$ denote the respective real and $\mat{D}\mpos_a$, $\mat{D}\mneg_b$ the respective DQ matrix representations.

\subsection{Optimization Problem}

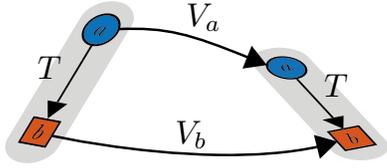
\begin{figure}
    \centering
    \resizebox{0.6\columnwidth}{!}{%
\begingroup%
  \makeatletter%
  \providecommand\color[2][]{%
    \errmessage{(Inkscape) Color is used for the text in Inkscape, but the package 'color.sty' is not loaded}%
    \renewcommand\color[2][]{}%
  }%
  \providecommand\transparent[1]{%
    \errmessage{(Inkscape) Transparency is used (non-zero) for the text in Inkscape, but the package 'transparent.sty' is not loaded}%
    \renewcommand\transparent[1]{}%
  }%
  \providecommand\rotatebox[2]{#2}%
  \newcommand*\fsize{\dimexpr\f@size pt\relax}%
  \newcommand*\lineheight[1]{\fontsize{\fsize}{#1\fsize}\selectfont}%
  \ifx\svgwidth\undefined%
    \setlength{\unitlength}{118.25025558bp}%
    \ifx\svgscale\undefined%
      \relax%
    \else%
      \setlength{\unitlength}{\unitlength * \real{\svgscale}}%
    \fi%
  \else%
    \setlength{\unitlength}{\svgwidth}%
  \fi%
  \global\let\svgwidth\undefined%
  \global\let\svgscale\undefined%
  \makeatother%
  \begin{picture}(1,0.45211106)%
    \lineheight{1}%
    \setlength\tabcolsep{0pt}%
    \put(0,0){\includegraphics[width=\unitlength,page=1]{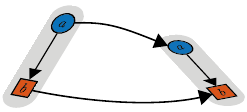}}%
    \put(0.44623413,0.0666195){\makebox(0,0)[lt]{\lineheight{1.25}\smash{\begin{tabular}[t]{l}$V_b$\end{tabular}}}}%
    \put(0.47453999,0.3666734){\makebox(0,0)[lt]{\lineheight{1.25}\smash{\begin{tabular}[t]{l}$V_a$\end{tabular}}}}%
    \put(0.82891406,0.1856647){\makebox(0,0)[lt]{\lineheight{1.25}\smash{\begin{tabular}[t]{l}$T$\end{tabular}}}}%
    \put(0.09342157,0.23218045){\makebox(0,0)[lt]{\lineheight{1.25}\smash{\begin{tabular}[t]{l}$T$\end{tabular}}}}%
  \end{picture}%
\endgroup%
    }
    \caption{
    The transformation graph for two sensors, sensor $a$ (\sensorcircle) and sensor $b$ (\sensorrect), at two consecutive time steps is shown.
    For this cycle, the unknown calibration $T$ and the estimated sensor motions $V_a$ and $V_b$ are illustrated.
    }
    \label{fig:transformation_cycle}
\end{figure}

In this section, the optimization problem derived in~\cite{horn2021online} is recapitulated.
The QCQP formulation is required for our sensitivity analysis and individual weighting in Section~\ref{sec:sensitivity_and_weighting}.
However, our approach can also be applied on any other QCQP formulation that accumulates sample-wise costs, e.g.~\cite{giamou2019certifiably}.

Fig.~\ref{fig:transformation_cycle} shows the hand-eye transformation cycle $V_a \circ T = T \circ V_b$ for a single motion pair.
Expressed as DQs, this yields $\quat{q}_a \,\quat{q}_T = \quat{q}_T \,\quat{q}_b$.
Vectorizing the DQs and using matrix-vector products for multiplications, as described in Section~\ref{sec:dual_quaternions}, leads to the linear formulation $\mat{M} \vec{x} = \vec{0}$ with $\mat{M} = \mat{Q}\mpos_a - \mat{Q}\mneg_b$ and $\vec{x} = \vectorize(\quat{q}_T) \in \mathds{R}^8$.
Weighted summation of the squared norm of this formulation over all motion pairs $i=1,\dots,n$ leads to the quadratic cost function
\begin{align}
\label{eq:cost_function}
    J(\vec{x}) = \vec{x}\trans \mat{Q} \vec{x} \quad
    \text{with } \mat{Q} = \sum_{i=1}^{n} w_i \, \mat{M}_i\trans \mat{M}_i
\end{align}
and sample weights $w_i \in \mathbb{R}$.
In contrast to~\cite{horn2021online}, we do not enforce $\sum_{i=1}^n w_i = 1$ in order to avoid numerical issues in case of a large number of samples.
Therefore, uniformly distributed weights with $w_i = 1$ are initially used before updating the weights for an improved conditioning, as described in Section~\ref{sec:sensitivity_and_weighting}. 

Two constraints must be enforced during optimization to ensure that $\vec{x}$ represents a valid unit DQ:
\begin{subequations}
\label{eq:constraints}
\begin{gather}
    \vec{g}_\mathrm{DQ}(\vec{x}) = 
    \begin{bmatrix}
        1 + \vec{x}\trans \mat{P}_\mathrm{r} \vec{x} \\
        \vec{x}\trans \mat{P}_\mathrm{d} \vec{x} \\
    \end{bmatrix}  
    \mbeq \vec{0} \pspace ,
    \\
    \text{with }
    \mat{P}_\mathrm{r} = 
    \begin{bmatrix}
        -\mat{I}_{4} & \mat{0}_{4} \\
        \mat{0}_{4} & \mat{0}_{4} \\
    \end{bmatrix} \pspace , \
    \mat{P}_\mathrm{d} = 
    \begin{bmatrix}
        \mat{0}_{4} & \mat{I}_{4} \\
        \mat{I}_{4} & \mat{0}_{4} \\
    \end{bmatrix}
    \pspace .
\end{gather}
\end{subequations}
In combination with the previously derived cost function~\eqref{eq:cost_function}, this leads to the QCQP formulation for hand-eye calibration using DQs:
\begin{subequations}
\begin{align}
    \!\min_{\vec{x}\in\mathbb{R}^{8}} & \quad J(\vec{x}) \pspace ,\\
    \text{w.r.t.} & \quad \vec{g}_\mathrm{DQ}(\vec{x}) \mbeq \vec{0} \pspace .
\end{align}
\end{subequations}

Due to the constraints, the optimization problem is non-convex, even though the cost function $J(\vec{x})$ is convex.
However, as described in~\cite{horn2021online}, the Semidefinite Programming~(SDP) relaxation of this QCQP problem can be used to obtain a certifiably globally optimal solution.
This globally optimal solution is the foundation for the approach described in this work.
For a more detailed explanation of the SDP relaxation, we refer to~\cite{horn2021online}.
\section{SENSITIVITY AND SAMPLE WEIGHTING}
\label{sec:sensitivity_and_weighting}

This section describes the main contribution of this work, visualized as green blocks in Fig.~\ref{fig:overview_graph}.
First, a method for analyzing the sensitivity of the cost function w.r.t. small deviations from the globally optimal solution is described in Section~\ref{sec:sensitivity}.
A relatively low sensitivity for a specific translation axis indicates that the rotation around this axis is overrepresented in the motion samples.
To reduce this imbalance, an approach for sample weighting based on the density of the rotation axes is described in Section~\ref{sec:density_weights}.
These weights are used in combination with the conditioning of the original optimization problem to obtain a balanced cost function in Section~\ref{sec:combined_weights}.

\subsection{Sensitivity and Conditioning}
\label{sec:sensitivity}

In this section, a method for estimating the sensitivity of the original, uniformly weighted cost function with weights $w_i = 1$ and cost matrix $\mat{Q}$ is described.
The sensitivity is estimated w.r.t. small deviations from the globally optimal solution $\sol{\quat{x}}$, here expressed as a DQ.
In contrast to the usual definition of sensitivity, we do not analyze the general influence of input uncertainties on the output.
Instead, we apply pure translational and rotational transformations on the dual quaternion solution $\sol{\quat{x}}$ and observe their influence on the costs.
Thus, the constraints~\eqref{eq:constraints} for valid unit dual quaternions are ensured for all samples.
The influence of these deviations is analyzed in the following.

\begin{proposition}[Linear Influence on Transformations]
The influence of transformations with small translations and small rotation angles on other transformations is approximately linear w.r.t. the translation norm and the rotation angle.
\end{proposition}

\begin{proof}
Applying a pure translation $T_\mathrm{t}(\alpha \, \vec{t})$ from~\eqref{eq:pure_translation} with scaling $\alpha \in \mathds{R}$ and translation vector $\vec{t} \in \mathds{R}^3$ on $\sol{\quat{x}} = \sol{\quat{r}} + \epsilon \, \sol{\quat{d}}$ yields the offset
\begin{subequations}
\begin{align}
    \sol{\quat{x}} \ T_\mathrm{t}(\alpha \, \vec{t}) - \sol{\quat{x}} 
    &= \left(\sol{\quat{r}} + \epsilon \, \sol{\quat{d}} \right) \, \left(\quatsv{1}{\vec{0}} + \epsilon \, \frac{1}{2} \quatsv{0}{\alpha \, \vec{t}} \right) - \sol{\quat{x}}  \\
    &= \left(\sol{\quat{r}} + \epsilon \, \sol{\quat{d}} \right) + \epsilon \, \left( \frac{\alpha}{2} \, \sol{\quat{r}} \quatsv{0}{\vec{t}} \right) - \sol{\quat{x}} \\
    &= \epsilon \, \frac{\alpha}{2} \, \sol{\quat{r}} \quatsv{0}{\vec{t}} \pspace ,
\end{align}
\end{subequations}
which is linear w.r.t. $\alpha$.
For small rotation angles, the pure rotation $T_\mathrm{r}(\varphi, \vec{n}) = \quatsv{\cos\left(\nicefrac{\varphi}{2}\right)}{\vec{n} \sin\left(\nicefrac{\varphi}{2}\right)}$ from~\eqref{eq:pure_rotation} is approximately $\quatsv{1}{\nicefrac{\varphi}{2} \, \vec{n}}$.
When applied to $\sol{\quat{x}}$, this yields the offset
\begin{subequations}
\begin{align}
    \sol{\quat{x}} \ T_\mathrm{r}(\varphi, \vec{n}) - \sol{\quat{x}}
    &\approx \left( \sol{\quat{r}} + \epsilon \, \sol{\quat{d}} \right) \, \left( \quatsv{1}{\nicefrac{\varphi}{2} \, \vec{n}} \right) - \sol{\quat{x}} \\
    &= 
    \begin{aligned}[t]
        \frac{\varphi}{2} 
        ( &\quatsv{-\sol{\vec{v}}_\quat{r} \vec{n}}{\ \sol{s}_\quat{r} \vec{n} + \sol{\vec{v}}_\quat{r} \times \vec{n}} + \nonumber\\
        \epsilon &\quatsv{-\sol{\vec{v}}_\quat{d} \vec{n}}{\ \sol{s}_\quat{d} \vec{n} + \sol{\vec{v}}_\quat{d} \times \vec{n}} ) \pspace ,
    \end{aligned}\\[-0.5cm]
\end{align}
\end{subequations}
which is linear w.r.t. $\varphi$.
\end{proof}

Since the cost function is quadratic and the constraints~\eqref{eq:constraints} are met for all deviations, this leads to the following conclusion:

\begin{corollary}[Quadratic Influence on the Costs]
\label{cor:quadratic_influence}
The influence on the costs of applying transformations with small translations and small rotation angles to the globally optimal solution $\sol{\quat{x}}$ is approximately quadratic.
\end{corollary}

This means that it is possible to estimate quadratic functions describing the influence of small translational and rotational deviations on the costs around the globally optimal solution.
These functions are defined as
\begin{align}
    S_\mathrm{t}(\vec{t}) = \vec{t}\trans \mat{S}_\mathrm{t} \vec{t} \pspace ,\quad
    S_\mathrm{r}(\vec{r}) = \vec{r}\trans \mat{S}_\mathrm{r} \vec{r} \pspace ,
\end{align}
with translations $\vec{t} \in \mathbb{R}^3$, rotations $\vec{r} = \varphi \vec{n} \in \mathbb{R}^3$ in axis-angle representation, and $\mat{S}_\mathrm{t}, \mat{S}_\mathrm{r} \in \mathbb{R}^{3 \times 3}$ with $\mat{S}_\mathrm{t} = \mat{S}_\mathrm{t}\trans$ and $\mat{S}_\mathrm{r} = \mat{S}_\mathrm{r}\trans$.
For these functions,
\begin{subequations}
\label{eq:cost_function_samples}
\begin{align}
    \vec{t}\trans \mat{S}_\mathrm{t} \vec{t} &\approx J(\sol{\quat{x}} \, T_\mathrm{t}(\vec{t})) - J(\sol{\quat{x}}) =: \Delta J_\mathrm{t}(\vec{t}) \pspace ,\\
    \vec{r}\trans \mat{S}_\mathrm{r} \vec{r} &\approx J(\sol{\quat{x}} \, T_\mathrm{r}(\varphi, \vec{n})) - J(\sol{\quat{x}}) =: \Delta J_\mathrm{r}(\vec{r}) \pspace ,
\end{align}
\end{subequations}
must hold for small translations and small rotation angles accordingly.
Since the estimation process is analogous for $\mat{S}_\mathrm{t}$ and $\mat{S}_\mathrm{r}$, the symbol $\bullet$ is used in the following as placeholder for either $\mathrm{t}$ or $\mathrm{r}$.

For a detailed evaluation of the respective translation and rotation conditioning, the matrices $\mat{S}_{\bullet}$ must be estimated first.
This is achieved by evaluating the cost function at six sample deviations
\begin{align}
\label{eq:sample_points}
    P = \{
    &\left[ 1 \ 0 \ 0 \right]\trans,
    \left[ 0 \ 1 \ 0 \right]\trans,
    \left[ 0 \ 0 \ 1 \right]\trans,\nonumber\\
    &\tfrac{1}{\sqrt{2}} \left[ 1 \ 1 \ 0 \right]\trans,
    \tfrac{1}{\sqrt{2}} \left[ 1 \ 0 \ 1 \right]\trans,
    \tfrac{1}{\sqrt{2}} \left[ 0 \ 1 \ 1 \right]\trans \}
\end{align}
with individual sample vectors $\vec{p}^{(k)}$, scaled by the translation $\delta_\mathrm{t} = \SI{0.1}{\meter}$ and the rotation $\delta_\mathrm{r} = \SI{0.1}{\deg}$, respectively.
Since the matrices $\mat{S}_{\bullet}$ are symmetric, these six samples are sufficient for estimating all entries.
The cost deviation samples are then given by
\begin{align}
\label{eq:cost_function_deviation}
    \Delta J_\mathrm{t}(\delta_\mathrm{t} \, \vec{p}^{(k)}) \pspace , \ 
    \Delta J_\mathrm{r}(\delta_\mathrm{r} \, \vec{p}^{(k)}) \pspace , \quad
    \text{for } k = 1, \dots, 6 \pspace .
\end{align}

Using the sample points~\eqref{eq:sample_points} and the respective cost function deviation~\eqref{eq:cost_function_deviation},
\begin{align}
    \delta_\bullet^2 \, \vec{p}^{(k)} \mat{S}_\bullet \vec{P^{(k)}} =
    \Delta J_{\bullet}(\delta_\bullet \, \vec{p}^{(k)}) \pspace , \ 
    \text{for } k = 1, \dots, 6 \pspace
\end{align}
is transformed to a system of linear equations and solved for the entries of $\mat{S}_{\bullet}$.

Using the estimated $\mat{S}_{\bullet}$, their eigenvalues $\lambda_{\bullet,1}, \lambda_{\bullet,2}, \lambda_{\bullet,3}$ with $|\lambda_{\bullet,1}| \leq |\lambda_{\bullet,2}| \leq |\lambda_{\bullet,3}|$, the corresponding eigenvectors $\vec{v}_{\bullet,1}, \vec{v}_{\bullet,2}, \vec{v}_{\bullet,3}$, and the condition number $c_{\bullet} = | \lambda_{\bullet,3} / \lambda_{\bullet,1} |$ with $c_{\bullet} \ge 1$ allow to draw conclusions about the influence of different deviations.
In our case, a lower eigenvalue indicates that deviations along the corresponding eigenvector have a smaller influence than deviations along the eigenvectors with larger eigenvalues, especially for large condition numbers.

For deriving more detailed information from the eigenvectors and condition numbers, the conclusions of Schmidt et al.~\cite{schmidt2003robust} are important.
They state that, for a large proportion of almost parallel rotation axes, the optimization problem is ill-conditioned since the translation along the rotation axis of each motion cannot be estimated.
Vice versa, this means that a large translation condition number $c_\mathrm{t}$ indicates that rotations around the axis $\vec{v}_{\mathrm{t},1}$ are overrepresented in the motion samples since deviations along this axis have a relatively low influence on the costs.
In contrast, a small condition number close to $1$ indicates more evenly distributed rotation axes.

This information can be employed as user feedback during calibration since the optimization and sensitivity estimation are fast enough to be calculated online while recording the motion samples.
In case of a large translation condition number $c_\mathrm{t}$, rotations with axes orthogonal to the eigenvector $\vec{v}_{\mathrm{t},1}$ should be added.
If the data has already been recorded ill-conditioned, the approach described in the following can help to reduce the imbalance.

\subsection{Density-based Sample Weighting}
\label{sec:density_weights}

In case of an ill-conditioned optimization problem, one possibility is to adjust the weights of the individual samples for reducing the influence of overrepresented rotations.
Instead of using a codebook for grouping or selecting samples as in \cite{schmidt2008data}, our approach uses the local density of the rotation axes for determining suitable sample weights.
A higher local density indicates that the data contain more samples with similar rotation axes and, thus, that these axes can be weighted lower to improve the cost function's condition.
Since the motions of sensors $a$ and $b$ are rigidly coupled, either the rotations of $a$ or $b$ can be used for calculating the density as described in the following.

First, the samples are divided into two groups using a threshold on the absolute rotation angles $|\varphi_i|$: samples with (almost) no rotation $(\mathrm{nr})$ and samples with rotation $(\mathrm{r})$.
This is necessary since the rotation axes for low rotations are less meaningful, i.e., the rotation axis for $\varphi_i = 0$ is arbitrary.
This results in $n^{(\mathrm{nr})}$ samples $\mat{M}_i^{(\mathrm{nr})}$ and $n^{(\mathrm{r})}$ samples $\mat{M}_i^{(\mathrm{r})}$ with rotation axes $\vec{n}_i^{(\mathrm{r})}$.

We define the local density of a rotation axis $\vec{n}_i^{(\mathrm{r})}$ as
\begin{align}
\label{eq:axes_density}
    \rho_i = \sum_{j=1}^{n^{(\mathrm{r})}} G_\mathrm{r} \left( d \left( \vec{n}_i^{(\mathrm{r})}, \vec{n}_j^{(\mathrm{r})} \right) \right) \pspace .
\end{align}
The distance function
\begin{align}
\label{eq:axes_distance}
    d(\vec{n}_i, \vec{n}_j) = \frac{\pi}{2} - \left| \arccos{\left( \vec{n}_i\trans \vec{n}_j \right)} - \frac{\pi}{2} \right|
\end{align}
yields the angle between two rotation axes, considering that the axes $\vec{n}$ and $-\vec{n}$ should have a distance of $0$, since they represent the same rotation with different orientations.
The zero-mean Gaussian kernel without normalization
\begin{align}
    \label{eq:density_gaussian_kernel}
    G_\mathrm{r}(x) = \exp{\left( \frac{x^2}{2 d_\mathrm{r}^2} \right)}
\end{align}
is used to reduce the influence of points with a larger distance without requiring a hard threshold.
The standard deviation $d_\mathrm{r}$ specifies the range for the density calculation.
By omitting the normalization factor, each axis has a minimum density of $1$ since it is always within distance $0$ to itself.

\begin{figure}
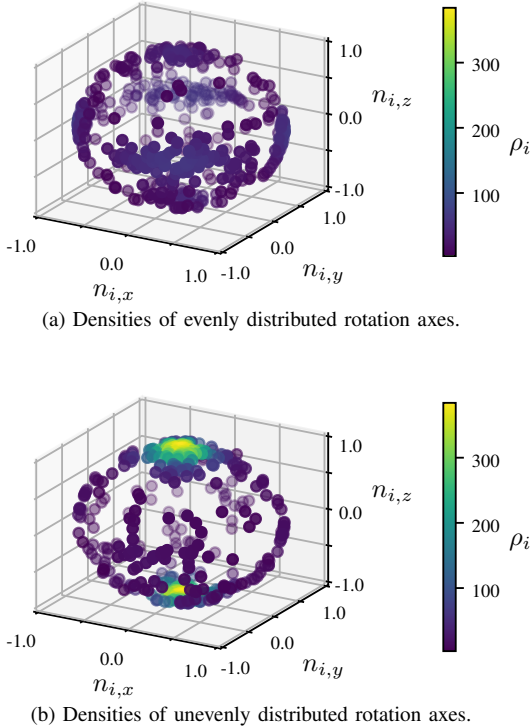

    \centering
    \subfloat[Densities of evenly distributed rotation axes.\label{fig:density:hill}]{%
        \centering%
        \def\svgwidth{0.9\columnwidth}%
        \input{images/density_even_img}%
    }
    \\
    \subfloat[Densities of unevenly distributed rotation axes.\label{fig:density:plane}]{%
        \centering%
        \def\svgwidth{0.9\columnwidth}%
        \input{images/density_uneven_img}%
    }
  \caption{
  The densities $\rho_i$ of datasets with \protect\subref{fig:density:hill} evenly and \protect\subref{fig:density:plane} unevenly distributed rotations are shown.
  Each axis $\vec{n}_i$ is represented as a point on the unit sphere; the color indicates the local density $\rho_i$.
  The densities of the unevenly distributed rotation axes clearly show a cluster at the $z$-axis, caused by mostly planar motion.
  }
  \label{fig:density}
\end{figure}

Fig.~\ref{fig:density} shows density examples for two differently distributed datasets.
As expected, the densities of the unevenly distributed rotation axes~(Fig.~\ref{fig:density:plane}) are significantly higher and more clustered than the densities of evenly distributed axes~(Fig.~\ref{fig:density:hill}).

The density-based weight for each point is calculated as
\begin{align}
    \label{eq:density_weight}
    w_{\rho, i} = \frac{1}{\sqrt{\rho_i}} \pspace .
\end{align}
The inverse square root ensures that samples with a higher density receive a lower weight, while still not overweighting single samples in relation to samples within larger clusters.
Our definition of density enables an iterative calculation of the densities and weights, e.g., in an online calibration scenario.

Finally, the no-rotation and rotation samples are combined to the weighted cost matrix $\mat{Q}_w$ using the weights
\begin{align}
    \label{eq:density_cost_matrix}
    w_i^{(\mathrm{nr})} = 1 \pspace , \quad 
    w_i^{(\mathrm{r})} = n^{(\mathrm{r})} \frac{w_{\rho, i}}{\Sigma_{\rho}}
\end{align}
with $\Sigma_{\rho} = \sum_i^{n^{(\mathrm{r})}} w_{\rho, i}$, which preserves the original ratio between no-rotation and rotation samples.

\subsection{Combined Sample Weighting}
\label{sec:combined_weights}

The weighted cost matrix $\mat{Q}_w$ results in a lower translation condition number due to a more even distribution of rotation axes.
However, this can also lead to an increase of the rotation condition number when the original problem is already well-conditioned.

To prevent this, a transition between the unweighted and the weighted cost matrices is achieved using a sensitivity-based weight in the form of a sigmoid function
\begin{align}
    \gamma = \frac{1}{1 + \exp{(s_\gamma \, (c_\gamma - c_\mathrm{t}))}} \pspace ,
\end{align}
which uses the translation condition number $c_\mathrm{t}$ of the unweighted problem.
The parameter $c_\gamma$ controls at which condition number $\gamma = 0.5$; the parameter $s_\gamma$ is used to adjust the slope of the transition between the original and the weighted cost matrix.
The combined cost matrix is then calculated by
\begin{align}
    \mat{Q}_\gamma = (1 - \gamma) \, \mat{Q} + \gamma \, \mat{Q}_w \pspace .
\end{align}
By this, the original cost matrix $\mat{Q}$ is preserved for lower values of $c_\mathrm{t}$, while for higher values of $c_\mathrm{t}$, the weighted cost matrix $\mat{Q}_w$ is used to a greater extent.

\section{EXPERIMENTS}
\label{sec:experiments}

Our approach is evaluated on artificially augmented data from the KITTI odometry dataset~\cite{geiger2012autonomous}.
This dataset provides sensor data and ground-truth motion data recorded with a vehicle driving in urban environments.
First, all non-planar components of the ground-truth motion are removed.
Afterward, non-planar components are added in a controlled manner by using a sine function with adjustable amplitude as an artificial elevation profile.
A higher amplitude leads to more non-planar motion and, thus, more evenly distributed rotation axes.
This makes it possible to create datasets with a known number of unevenly distributed $n^{(\mathrm{uneven})}$ and evenly distributed $n^{(\mathrm{even})}$ rotation axes while retaining the general motion of the vehicle.
A fixed $n^{(\mathrm{even})} = 100$ is used for all experiments.
The motion of the second sensor is then created by transforming the augmented ground-truth motion with an artificial calibration.
Finally, Gaussian distributed noise is added to each transformation to simulate the behavior of visual odometry.
The amount is chosen proportional to its absolute translation, with standard deviations $\sigma_\mathrm{r}$ in \si{\deg\per\meter} for rotation and $\sigma_\mathrm{t}$ in~\si{\%} for translation.

We compare our method to the VQ-based approach proposed by Schmidt et al.~\cite{schmidt2008data}.
For this, we use normalized rotation axes as representation for rotations and $k$-means clustering~\cite{macqueen1967some} as VQ algorithm.
The authors of~\cite{schmidt2008data} propose a preprocessing step in which samples with low rotation magnitudes are discarded.
For a fair comparison, all no-rotation samples are used as described in Section~\ref{sec:density_weights} since the VQ-based performance significantly drops in our experiments when discarding these samples.
This is most likely caused by the fact that a larger number of samples, in general, can help to compensate for noise, even though they contain no rotational information.

The rotation and translation errors for an estimated calibration $\sol{\quat{q}}$ and known ground-truth $\quat{q}_T$ are calculated by
\begin{subequations}
\begin{align}
    \varepsilon_\mathrm{r} &= 2 \arccos(\vec{q}_{\varepsilon,1}) \pspace ,\\
    \varepsilon_\mathrm{t} &= \norm{2 \, \quat{q}_{\varepsilon,d} \, \quat{q}^*_{\varepsilon,r}} 
\end{align}
\end{subequations}
with $\quat{q}_\varepsilon = \quat{q}_T^{-1} \, \sol{\quat{q}}$ and $\vec{q}_\varepsilon = \vectorize(\quat{q}_\varepsilon)$.
All results are averaged over $10$ runs.

\subsection{Sensitivity}
\label{sec:experiments:sensitivity}

First, it is verified that the translation condition number $c_\mathrm{t}$ can be utilized to detect an overrepresented axis.
For this, $c_\mathrm{t}$ is calculated for a gradually increasing proportion of unevenly distributed samples $n^{(\mathrm{uneven})}$.
The results shown in Fig.~\ref{fig:sensitivity} clearly demonstrate that an increasing $n^{(\mathrm{uneven})}$ directly affects $c_\mathrm{t}$.
Thus, $c_\mathrm{t}$ is a useful measure for assessing the relative amount of planar motion.

\begin{figure}
    \centering
\begin{tikzpicture}

\definecolor{darkgray176}{RGB}{176,176,176}
\definecolor{steelblue31119180}{RGB}{31,119,180}

\begin{axis}[
width=0.8\columnwidth,
height = 0.3\linewidth,
scale only axis,
label style={font=\footnotesize},
tick align=inside,
tick pos=left,
ticklabel style={font=\footnotesize},
x grid style={darkgray176},
xlabel={\(\displaystyle n^{\mathrm{(uneven)}}\)},
xmin=-395, xmax=10495,
xtick style={color=black},
xlabel style={yshift=0.3em},
y grid style={darkgray176},
ylabel={\(\displaystyle c_\mathrm{t}\)},
ymin=-0.383186830805022, ymax=44.8546520067181,
ytick style={color=black},
ylabel style={yshift=-0.3em},
scaled ticks=false
]
\input{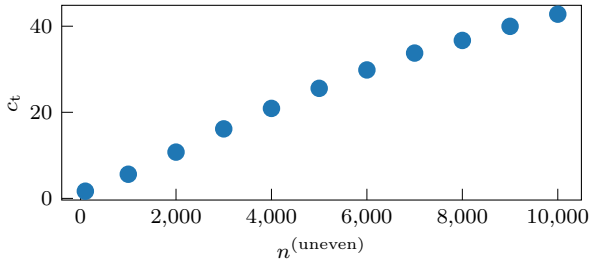}
\end{axis}

\end{tikzpicture}
    \caption{
    The translation condition numbers $c_\mathrm{t}$ for multiple values of $n^{(\mathrm{uneven})}$ are shown.
    The experiment clearly shows that $c_\mathrm{t}$ is a useful indicator of an overrepresented rotation axis.
    }
    \label{fig:sensitivity}
\end{figure}

\subsection{Parameter Selection}
\label{sec:experiments:parameters}

Before evaluating the influence of the VQ-based and our density-based method on the calibration error, the hyperparameters of both approaches must be selected.
Schmidt et al.~\cite{schmidt2008data} recommend using a relative codebook size $k_\mathrm{rel} = \nicefrac{k_\mathrm{clusters}}{n^{(\mathrm{r})}}$.
For our density-based weight, the density range $d_\mathrm{r}$ must be configured.

\begin{figure}
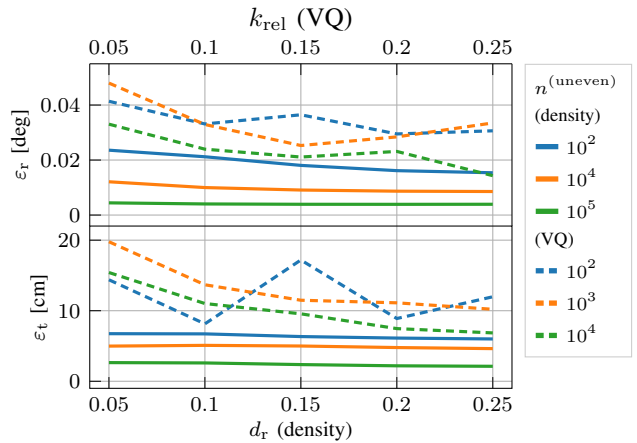

    \centering
\begin{tikzpicture}

\definecolor{darkgray176}{RGB}{176,176,176}
\definecolor{darkorange25512714}{RGB}{255,127,14}
\definecolor{forestgreen4416044}{RGB}{44,160,44}
\definecolor{lightgray204}{RGB}{204,204,204}
\definecolor{steelblue31119180}{RGB}{31,119,180}

\newlength{\paramboxwidth}
\setlength{\paramboxwidth}{0.65\columnwidth}
\newlength{\paramboxheight}
\setlength{\paramboxheight}{0.25\columnwidth}

\pgfplotsset{
	legend image code/.code={
		\draw[mark repeat=2,mark phase=2]
		plot coordinates {
			(0cm,0cm)
			(0.15cm,0cm)
			(0.3cm,0cm)
		};%
	}
}

\pgfplotsset{
    every non boxed x axis/.append style={x axis line style=-},
    every non boxed y axis/.append style={y axis line style=-}
}

\begin{groupplot}[group style={group size=1 by 2, vertical sep=0em}]

\nextgroupplot[
width=\paramboxwidth,
height=\paramboxheight,
scale only axis,
legend columns=1,
legend cell align={left},
legend style={
  legend pos=outer north east,
  fill opacity=0.8,
  draw opacity=1,
  text opacity=1,
  anchor=north west,
  draw=lightgray204,
  font=\scriptsize
},
legend image post style={line width=2pt},
label style={font=\footnotesize},
axis x line=top,
tick align=inside,
ticklabel style={font=\footnotesize},
scaled ticks=false,
x grid style={darkgray176},
xlabel={\(\displaystyle k_{\mathrm{rel}}\) (VQ)},
xlabel style={yshift=-0.3em},
xmajorgrids,
xmin=0.04, xmax=0.26,
xtick pos=right,
xtick style={color=black},
x tick label style={
	/pgf/number format/fixed,
	/pgf/number format/precision=2
},
y grid style={darkgray176},
ylabel={\(\displaystyle \varepsilon_\mathrm{r}\) [deg]},
ylabel style={yshift=-0.3em},
ymajorgrids,
ymin=-0.004, ymax=0.055,
ytick pos=left,
ytick style={color=black},
ytick={0,0.02,0.04},
y tick label style={
	/pgf/number format/fixed,
	/pgf/number format/precision=2
},
]
\input{plots/data/parameters_rot.tex}

\nextgroupplot[
width=\paramboxwidth,
height=\paramboxheight,
scale only axis,
label style={font=\footnotesize},
tick align=inside,
tick pos=left,
ticklabel style={font=\footnotesize},
scaled ticks=false,
x grid style={darkgray176},
xlabel={\(\displaystyle d_\mathrm{r}\) (density)},
xlabel style={yshift=0.3em},
xmajorgrids,
xmin=0.04, xmax=0.26,
xtick style={color=black},
x tick label style={
	/pgf/number format/fixed,
	/pgf/number format/precision=2
},
y grid style={darkgray176},
ylabel={\(\displaystyle \varepsilon_\mathrm{t}\)  [cm]},
ylabel style={yshift=-0.3em},
ymajorgrids,
ymin=-1, ymax=22.0,
ytick style={color=black},
]
\input{plots/data/parameters_trans.tex}

\end{groupplot}

\end{tikzpicture}
    \caption{
    Results of a parameter sweep of $d_\mathrm{r}$ for our density-based method and $k_\mathrm{rel}$ for the VQ-based method using multiple values for $n^{(\mathrm{uneven})}$ are shown.
    The best trade-off between $\epsilon_\mathrm{r}$ and $\epsilon_\mathrm{t}$ is achieved at $d_\mathrm{r} = 0.2$ and $k_\mathrm{rel} = 0.2$, respectively.
    Compared to our method, the VQ-based method shows a larger performance variation for different parameters.
    }
    \label{fig:parameters}
\end{figure}

The errors for multiple configurations of both methods using noise levels of $\sigma_\mathrm{r} = \SI{0.005}{\deg\per\meter}$ and $\sigma_\mathrm{t} = \SI{0.025}{\%}$ are illustrated in Fig.~\ref{fig:parameters}.
Based on these results, $d_\mathrm{r} = 0.2$ and $k_\mathrm{rel} = 0.2$ are selected.
The comparison between the VQ-based and our method further demonstrates that, in contrast to our approach, their performance varies greatly for different parameters $k_\mathrm{rel}$.
This indicates that our density-based method is much easier to parameterize than the VQ-based method.

For both methods, a rotation threshold of $\SI{0.1}{\deg}$ for separating no-rotation and rotation samples has achieved the best results.
The parameters for the combined sample weighting were set to $c_\gamma = 15$ and $s_\gamma = 0.2$.

\subsection{Calibration Errors}
\label{sec:experiments:calibration}

\begin{table*}[t]
    \caption{Comparison between uniform weighting, VQ-based sample selection, and density-based weighting.}
    \label{tab:calibration}
    \begin{center}
        \begin{tabular}{rO{2.2}O{1.3}|O{2.2}O{1.3}|O{2.2}O{1.3}|O{2.2}O{1.3}|O{2.2}O{1.3}|O{2.2}O{1.3}}
    \toprule
    Noise
    & \multicolumn{6}{c|}{$\sigma_\mathrm{r} = \SI{0.01}{\deg\per\meter}$, \, $\sigma_\mathrm{t} = \SI{0.05}{\%}$}
    & \multicolumn{6}{c}{$\sigma_\mathrm{r} = \SI{0.02}{\deg\per\meter}$, \, $\sigma_\mathrm{t} = \SI{0.1}{\%}$} \\
    
    $n^{(\mathrm{uneven})}$
    & \multicolumn{2}{c|}{$100$} & \multicolumn{2}{c|}{$1000$} & \multicolumn{2}{c|}{$10000$}
    & \multicolumn{2}{c|}{$100$} & \multicolumn{2}{c|}{$1000$} & \multicolumn{2}{c}{$10000$} \\
    
    \backslashbox{Method}{Errors}
    & $\varepsilon_\mathrm{t}$ \ [$\si{\centi\metre}$] & $\varepsilon_\mathrm{r}$ \ [$\si{\degree}$]
    & $\varepsilon_\mathrm{t}$ \ [$\si{\centi\metre}$] & $\varepsilon_\mathrm{r}$ \ [$\si{\degree}$]
    & $\varepsilon_\mathrm{t}$ \ [$\si{\centi\metre}$] & $\varepsilon_\mathrm{r}$ \ [$\si{\degree}$]
    & $\varepsilon_\mathrm{t}$ \ [$\si{\centi\metre}$] & $\varepsilon_\mathrm{r}$ \ [$\si{\degree}$]
    & $\varepsilon_\mathrm{t}$ \ [$\si{\centi\metre}$] & $\varepsilon_\mathrm{r}$ \ [$\si{\degree}$]
    & $\varepsilon_\mathrm{t}$ \ [$\si{\centi\metre}$] & $\varepsilon_\mathrm{r}$ \ [$\si{\degree}$] \\
    
    \midrule

    Uniform Weighting
        & \B 12.68 & \B 0.031 &    13.07 &    0.019 &     8.96 & \B 0.010
        &    43.03 & \B 0.081 &    21.86 & \B 0.041 &    17.43 & \B 0.021 \\
    
    Schmidt et al.~\cite{schmidt2008data}
        &    27.60 &    0.049 &    20.89 &    0.049 &    17.36 &    0.031
        &    74.32 &    0.180 &    62.25 &    0.131 &    39.67 &    0.087 \\
        
    Ours
        &    12.70 & \B 0.031 & \B 12.94 & \B 0.018 & \B  6.18 & \B 0.010
        & \B 42.72 & \B 0.081 & \B 21.19 & \B 0.041 & \B 11.72 &    0.026 \\

    \bottomrule    
\end{tabular}

    \end{center}
\end{table*}

Finally, the impact of both methods on the calibration error is evaluated using the previously selected parameters.
For this, the calibration is performed on multiple noise configurations and values of $n^{(\mathrm{uneven})}$.
Fig.~\ref{fig:calibration} compares the original results with uniformly distributed weights and the results after applying the respective method.
The VQ-based method induces larger errors in almost all cases.
Since their method was not designed for data with relatively high noise, removing samples can lead to larger errors if a lot of samples with relatively high noise are left.
In contrast, our proposed method leads to almost no increase in the rotation error.
Furthermore, the translation estimation is significantly improved for data with a large translation condition number $c_\mathrm{t}$.
For data with a small $c_\mathrm{t}$, the results remain unchanged, as intended by the sensitivity-based weighting.

\begin{figure}
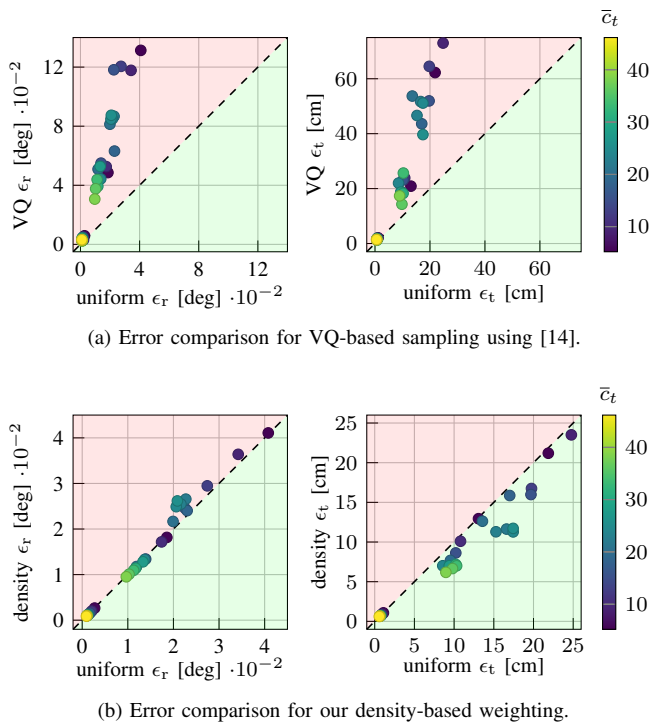

\centering
\newlength{\errorboxsize}
\setlength{\errorboxsize}{0.33\columnwidth}
\subfloat[Error comparison for VQ-based sampling using \cite{schmidt2008data}.\label{fig:calibration:kmeans}]{%
    \centering%
\begin{tikzpicture}

\begin{groupplot}[group style={group size=2 by 1, horizontal sep=3em}]

\nextgroupplot[
width=\errorboxsize,
height=\errorboxsize,
scale only axis,
label style={font=\footnotesize},
tick align=inside,
tick pos=left,
ticklabel style={font=\footnotesize},
scaled ticks=false,
grid=both,
xlabel={uniform \(\displaystyle \epsilon_\mathrm{r}\) [deg] $\cdot 10^{-2}$},
xmin=-0.005, xmax=0.14,
xtick style={color=black},
xlabel style={yshift=0.3em},
xtick={0,0.04,0.08,0.12},
xticklabels={
   0, 4, 8, 12
},
ylabel={VQ \(\displaystyle \epsilon_\mathrm{r}\) [deg] $\cdot 10^{-2}$},
ylabel style={yshift=-0.3em},
ymin=-0.005, ymax=0.14,
ytick style={color=black},
ytick={0,0.04,0.08,0.12},
yticklabels={
   0, 4, 8, 12
},
]
\input{plots/data/errors_kmeans_rot.tex}

\draw[name path=bisector, semithick, black, dashed]
(rel axis cs:0,0) -- (rel axis cs:1,1);
\path[name path=above] (rel axis cs:0,1) -- (rel axis cs:1,1);
\path[name path=below] (rel axis cs:0,0) -- (rel axis cs:1,0);
\addplot [green, opacity=0.1] fill between[of=bisector and below];
\addplot [red, opacity=0.1] fill between[of=above and bisector];

\nextgroupplot[
width=\errorboxsize,
height=\errorboxsize,
scale only axis,
colorbar,
colormap/viridis,
colorbar style={
    title={\(\displaystyle \overline{c}_t\)},
    title style={yshift=-0.5em, font=\footnotesize},
    width=0.5em,
},
label style={font=\footnotesize},
tick align=inside,
tick pos=left,
ticklabel style={font=\footnotesize},
grid=both,
xlabel={uniform \(\displaystyle \epsilon_\mathrm{t}\) [cm]},
xmin=-3.0, xmax=75.0,
xtick style={color=black},
xlabel style={yshift=0.3em},
xtick={0,20,40,60},
ylabel={VQ \(\displaystyle \epsilon_\mathrm{t}\) [cm]},
ylabel style={yshift=-0.3em},
ymin=-3.0, ymax=75.0,
ytick style={color=black},
ytick={0,20,40,60},
]
\input{plots/data/errors_kmeans_trans.tex}

\draw[name path=bisector, semithick, black, dashed]
(rel axis cs:0,0) -- (rel axis cs:1,1);
\path[name path=above] (rel axis cs:0,1) -- (rel axis cs:1,1);
\path[name path=below] (rel axis cs:0,0) -- (rel axis cs:1,0);
\addplot [green, opacity=0.1] fill between[of=bisector and below];
\addplot [red, opacity=0.1] fill between[of=above and bisector];

\end{groupplot}
\end{tikzpicture}%
}%
\\%
\subfloat[Error comparison for our density-based weighting.\label{fig:calibration:density}]{%
    \centering%
\begin{tikzpicture}

\begin{groupplot}[group style={group size=2 by 1, horizontal sep=3em}]

\nextgroupplot[
width=\errorboxsize,
height=\errorboxsize,
scale only axis,
label style={font=\footnotesize},
tick align=inside,
tick pos=left,
ticklabel style={font=\footnotesize},
scaled ticks=false,
grid=both,
xlabel={uniform \(\displaystyle \epsilon_\mathrm{r}\) [deg] $\cdot 10^{-2}$},
xmin=-0.002, xmax=0.045,
xtick style={color=black},
xlabel style={yshift=0.5em},
xtick={0,0.01,0.02,0.03,0.04},
xticklabels={
   0, 1, 2, 3, 4
},
ylabel={density \(\displaystyle \epsilon_\mathrm{r}\) [deg] $\cdot 10^{-2}$},
ylabel style={yshift=-0.5em},
ymin=-0.002, ymax=0.045,
ytick style={color=black},
ylabel style={yshift=0.6em},
ytick={0,0.01,0.02,0.03,0.04},
yticklabels={
   0, 1, 2, 3, 4
},
]
\input{plots/data/errors_density_rot.tex}

\draw[name path=bisector, semithick, black, dashed]
(rel axis cs:0,0) -- (rel axis cs:1,1);
\path[name path=above] (rel axis cs:0,1) -- (rel axis cs:1,1);
\path[name path=below] (rel axis cs:0,0) -- (rel axis cs:1,0);
\addplot [green, opacity=0.1] fill between[of=bisector and below];
\addplot [red, opacity=0.1] fill between[of=above and bisector];

\nextgroupplot[
width=\errorboxsize,
height=\errorboxsize,
scale only axis,
colorbar,
colormap/viridis,
colorbar style={
    title={\(\displaystyle \overline{c}_t\)},
    title style={yshift=-0.5em, font=\footnotesize},
    width=0.5em,
    xshift=-0.5em,
},
label style={font=\footnotesize},
tick align=inside,
tick pos=left,
ticklabel style={font=\footnotesize},
grid=both,
xlabel={uniform \(\displaystyle \epsilon_\mathrm{t}\) [cm]},
xmin=-1.0, xmax=26.0,
xtick style={color=black},
xlabel style={yshift=0.5em},
xtick={0,5,10,15,20,25},
ylabel={density \(\displaystyle \epsilon_\mathrm{t}\) [cm]},
ylabel style={yshift=-0.5em},
ymin=-1.0, ymax=26.0,
ytick style={color=black},
ytick={0,5,10,15,20,25},
]
\input{plots/data/errors_density_trans.tex}

\draw[name path=bisector, semithick, black, dashed]
(rel axis cs:0,0) -- (rel axis cs:1,1);
\path[name path=above] (rel axis cs:0,1) -- (rel axis cs:1,1);
\path[name path=below] (rel axis cs:0,0) -- (rel axis cs:1,0);
\addplot [green, opacity=0.1] fill between[of=bisector and below];
\addplot [red, opacity=0.1] fill between[of=above and bisector];

\end{groupplot}
\end{tikzpicture}%
}%
\caption{
    A comparison between the original rotation and translation errors using uniform sample weighting and the errors after applying the respective method is shown for $n^{(\mathrm{uneven})}$ from $1000$ to $10000$ and multiple noise levels.
    The color of each point indicates the average translation condition number $\overline{c}_t$ for this combination of $n^{(\mathrm{uneven})}$ and noise level.
    Points on the dashed bisector indicate no change, points below the bisector indicate a reduced error, and points over the bisector indicate a larger error caused by the respective method.
}
\label{fig:calibration}
\end{figure}

Table~\ref{tab:calibration} shows quantitative examples for two noise levels and multiple values of $n^{(\mathrm{uneven})}$.
Especially for more noise and larger $n^{(\mathrm{uneven})}$, the results are significantly improved by our proposed method.
For smaller $n^{(\mathrm{uneven})}$, i.e., more evenly distributed rotation axes, our method does not negatively affect the result.
Furthermore, the rotational accuracies of our method and the baseline with uniform weighting are almost identical in all cases.
\section{CONCLUSION}

We have presented a density-based method for weighting hand-eye calibration samples in order to reduce the negative influence of an overrepresented rotation axis.
Our evaluation on artificially augmented motion data shows that our method is able to outperform the state of the art for the use case of automated vehicles.
Further, we have derived an estimation of the translation conditioning and the overrepresented rotation axis.
This information can be employed as user feedback during the calibration process to prevent ill-conditioned data already during recording.

Especially when calibrating sensors of complex systems like automated vehicles, it is helpful if the calibrating user is not required to have detailed system knowledge about the calibration algorithm.
Our approach for providing user feedback and for an automated weighting of samples is one step further toward this goal.



{\small
\bibliographystyle{IEEEtran}
\bibliography{mybibfile}

\begin{thebibliography}{10}
\providecommand{\url}[1]{#1}
\csname url@rmstyle\endcsname
\providecommand{\newblock}{\relax}
\providecommand{\bibinfo}[2]{#2}
\providecommand\BIBentrySTDinterwordspacing{\spaceskip=0pt\relax}
\providecommand\BIBentryALTinterwordstretchfactor{4}
\providecommand\BIBentryALTinterwordspacing{\spaceskip=\fontdimen2\font plus
\BIBentryALTinterwordstretchfactor\fontdimen3\font minus
  \fontdimen4\font\relax}
\providecommand\BIBforeignlanguage[2]{{%
\expandafter\ifx\csname l@#1\endcsname\relax
\typeout{** WARNING: IEEEtran.bst: No hyphenation pattern has been}%
\typeout{** loaded for the language `#1'. Using the pattern for}%
\typeout{** the default language instead.}%
\else
\language=\csname l@#1\endcsname
\fi
#2}}

\bibitem{daniilidis1999hand}
K.~Daniilidis, ``Hand-eye calibration using dual quaternions,'' \emph{The
  International Journal of Robotics Research}, vol.~18, no.~3, pp. 286--298,
  1999.

\bibitem{andreff2001robot}
N.~Andreff, R.~Horaud, and B.~Espiau, ``Robot hand-eye calibration using
  structure-from-motion,'' \emph{The International Journal of Robotics
  Research}, vol.~20, no.~3, pp. 228--248, 2001.

\bibitem{brookshire2013extrinsic}
J.~Brookshire and S.~Teller, ``Extrinsic calibration from per-sensor
  egomotion,'' in \emph{Robotics: Science and Systems VIII}, 2013.

\bibitem{horn2021online}
M.~Horn, T.~Wodtko, M.~Buchholz, and K.~Dietmayer, ``Online extrinsic
  calibration based on per-sensor ego-motion using dual quaternions,''
  \emph{IEEE Robotics and Automation Letters}, vol.~6, no.~2, pp. 982--989,
  2021.

\bibitem{tsai1989new}
R.~Y. Tsai and R.~K. Lenz, ``A new technique for fully autonomous and efficient
  3d robotics hand/eye calibration,'' \emph{IEEE Transactions on Robotics and
  Automation}, vol.~5, pp. 345--358, 1989.

\bibitem{schmidt2003robust}
J.~Schmidt, F.~Vogt, and H.~Niemann, ``Robust hand-eye calibration of an
  endoscopic surgery robot using dual quaternions,'' in \emph{Joint Pattern
  Recognition Symposium}, 2003, pp. 548--556.

\bibitem{shi2005approach}
F.~Shi, J.~Wang, and Y.~Liu, ``An approach to improve online hand-eye
  calibration,'' in \emph{Iberian Conference on Pattern Recognition and Image
  Analysis}, vol. 3522, 2005, pp. 647--655.

\bibitem{giamou2019certifiably}
M.~Giamou, Z.~Ma, V.~Peretroukhin, and J.~Kelly, ``Certifiably globally optimal
  extrinsic calibration from per-sensor egomotion,'' \emph{IEEE Robotics and
  Automation Letters}, vol.~4, no.~2, pp. 367--374, 2019.

\bibitem{zhang2005adaptive}
J.~Zhang, S.~Fanhuai, and Y.~Liu, ``An adaptive selection of motion for online
  hand-eye calibration,'' in \emph{Australasian Joint Conference on Artificial
  Intelligence}, 2005, pp. 520--529.

\bibitem{li2018simultaneous}
W.~Li, M.~Dong, N.~Lu, X.~Lou, and P.~Sun, ``Simultaneous robot-world and
  hand-eye calibration without a calibration object,'' \emph{Sensors}, vol.~18,
  2018.

\bibitem{liu2019robust}
J.~Liu, J.~Wu, and X.~Li, ``Robust and accurate hand-eye calibration method
  based on schur matric decomposition,'' \emph{Sensors}, vol.~19, 2019.

\bibitem{mair2011spatio}
E.~Mair, M.~Fleps, M.~Suppa, and D.~Burschka, ``Spatio-temporal initialization
  for imu to camera registration,'' in \emph{IEEE International Conference on
  Robotics and Biomimetics}, 2011, pp. 557--564.

\bibitem{samant2019robust}
C.~Samant, A.~Habed, M.~de~Mathelin, and L.~Goffin, ``Robust hand-eye
  calibration via iteratively re-weighted rank-constrained semi-definite
  programming,'' in \emph{IEEE/RSJ International Conference on Intelligent
  Robots and Systems (IROS)}, 2019, pp. 4482--4489.

\bibitem{schmidt2008data}
J.~Schmidt and H.~Niemann, ``Data selection for hand-eye calibration: A vector
  quantization approach,'' \emph{The International Journal of Robotics
  Research}, vol.~27, pp. 1027--1053, 2008.

\bibitem{siciliano2009robotics}
B.~Siciliano, L.~Sciavicco, L.~Villani, and G.~Oriolo, \emph{Robotics:
  Modeling, Planning, and Control}.\hskip 1em plus 0.5em minus 0.4em\relax
  Springer, 2009.

\bibitem{mccarthy1990introduction}
J.~M. McCarthy, \emph{Introduction to Theoretical Kinematics}.\hskip 1em plus
  0.5em minus 0.4em\relax MIT Press, 1990.

\bibitem{geiger2012autonomous}
A.~Geiger, P.~Lenz, and R.~Urtasun, ``Are we ready for autonomous driving? the
  kitti vision benchmark suite,'' in \emph{IEEE/CVF Conference on Computer
  Vision and Pattern Recognition (CVPR)}, 2012, pp. 3354--3361.

\bibitem{macqueen1967some}
J.~MacQueen \emph{et~al.}, ``Some methods for classification and analysis of
  multivariate observations,'' in \emph{Proceedings of the fifth Berkeley
  symposium on mathematical statistics and probability}, vol.~1, no.~14.\hskip
  1em plus 0.5em minus 0.4em\relax Oakland, CA, USA, 1967, pp. 281--297.

\end{thebibliography}
}

\end{document}